\documentclass[accepted]{uai2025} 
                        
\usepackage[american]{babel}
\PassOptionsToPackage{colorlinks,citecolor=blue,urlcolor=blue,linkcolor=blue,linktocpage=true}{hyperref} 

\usepackage[square,numbers,sort]{natbib}
\usepackage{mathtools} 
\usepackage{booktabs} 
\usepackage{tikz} 

\usepackage{wrapfig}
\usepackage{comment}
\usepackage{algorithmic,algorithm}

\usepackage{microtype}

\usepackage{multirow,makecell}
\usepackage{amsmath,amssymb,amsthm,bm}
\usepackage{mathtools}
\usepackage{enumerate}
\usepackage{enumitem}
\usepackage{subfigure}
\usepackage{color}

\usepackage[utf8]{inputenc} 
\usepackage[T1]{fontenc}    
\usepackage{hyperref}       
\usepackage{url}            
\usepackage{booktabs}       
\usepackage{amsfonts}       
\usepackage{nicefrac}       
\usepackage{microtype}      
\usepackage{xcolor}         

\newtheorem{theorem}{Theorem}
\newtheorem{lemma}[theorem]{Lemma}
\newtheorem{definition}[theorem]{Definition}
\newtheorem{remark}[theorem]{Remark}

\newtheorem{assumption}[theorem]{Assumption}

\newcommand{\argmin}{\mathop{\mathrm{argmin}}}
\newcommand{\argmax}{\mathop{\mathrm{argmax}}}

\def\E{\mathbb{E}}

\def\1{\mathbbm{1}}

\def\cF{\mathcal{F}}

\def\cH{\mathcal{H}}

\def\cX{\mathcal{X}}

\title{Bayesian Optimization with Inexact Acquisition:\\Is Random Grid Search Sufficient?}

\author[1]{Hwanwoo Kim}
\author[2]{Chong Liu}
\author[3]{Yuxin Chen}
\affil[1]{%
    Department of Statistical Science\\
Duke University\\
Durham, NC, USA
}
\affil[2]{%
    Department of Computer Science\\
University at Albany, State University of New York\\
Albany, NY, USA
}
\affil[3]{%
    Department of Computer Science\\
University of Chicago\\
Chicago, IL, USA
  }
  
\begin{document}
\maketitle

\begin{abstract}
Bayesian optimization (BO) is a widely used iterative algorithm for optimizing black-box functions. Each iteration requires maximizing an acquisition function, such as the upper confidence bound (UCB) or a sample path from the Gaussian process (GP) posterior, as in Thompson sampling (TS). However, finding an exact solution to these maximization problems is often intractable and computationally expensive. Reflecting realistic scenarios, this paper investigates the effect of using inexact acquisition function maximizers in BO. Defining a measure of inaccuracy in acquisition solutions, we establish cumulative regret bounds for both GP-UCB and GP-TS based on imperfect acquisition function maximizers. Our results show that under appropriate conditions on accumulated inaccuracy, BO algorithms with inexact maximizers can still achieve sublinear cumulative regret. Motivated by such findings, we provide both theoretical justification and numerical validation for random grid search as an effective and computationally efficient acquisition function solver. 
\end{abstract}


\section{Introduction}\label{sec:intro}
Bayesian optimization (BO) is a class of machine learning-based black-box optimization strategies for finding a global optimum of a real-valued function $f$. Typically, $f$ is a function whose analytical form is unavailable, but it can be evaluated with a substantial computational cost. Due to its black-box nature, BO has demonstrated its efficacy across a broad spectrum of practical applications, such as tuning hyperparameters in deep learning \citep{wu2020practical,turner2021bayesian,kandasamy2020tuning}, searching for neural network architectures \citep{kandasamy2018neural,zhou2019bayesnas,white2021bananas}, designing materials \citep{frazier2016bayesian,zhang2020bayesian,lei2021bayesian}, and discovering new drugs \citep{korovina2020chembo,bellamy2022batched,colliandre2023bayesian,yang2019batched}.

BO follows a sequential decision-making process, where the outcome of each iteration informs the selection of the next evaluation point. After each step, the surrogate model is updated, thereby guiding the selection of the next point for function evaluation. This process is built around two key components: (1) updating a Gaussian process (GP) surrogate model and (2) optimizing the acquisition function to decide the next evaluation point. The Gaussian process surrogate is favored for modeling the objective function $f$ due to its capability to iteratively incorporate all previous observations, possibly corrupted by noise. The acquisition function $\alpha_t$ directs the optimization by solving for the next evaluation location, denoted by $x_t$,
\begin{align*}
x_t =\argmax_{x \in \cX} \alpha_t (x).
\end{align*}
This function balances the exploration of the search space with the exploitation of the accurately estimated surrogate model.  
In BO, this is referred to as \textit{acquisition (inner) optimization} \cite{wilson2018maximizing}, while optimization of the whole objective function is termed \textit{outer optimization}. 
Various strategies for designing the acquisition function exist, with popular options including Expected Improvement (EI) \citep{jones1998efficient}, Knowledge Gradient (KG) \citep{frazier2009knowledge}, Probability of Improvement (PI) \citep{kushner1964new}, Upper Confidence Bound (UCB) \citep{srinivas2010gaussian}, and Thompson Sampling (TS) \citep{thompson1933likelihood}, each offering a unique approach to manage the exploration-exploitation balance.

\subsection{Inexact Acquisition Maximization}

Optimizing the acquisition function $\alpha_t$ is generally considered more straightforward and computationally efficient than directly optimizing the objective $f$, largely because $\alpha_t$ is defined through a Gaussian process surrogate that offers a tractable approximation to the otherwise complex, black-box function. However, theoretical approaches often require an \textit{exact solution} $x_t$ for establishing regret bounds, a task that becomes particularly difficult in practice due to the possible non-convex nature of the acquisition function. This situation raises an important question:
\begingroup
\begin{quote}
What are the implications of 
\textit{inexact} acquisition function maximization in BO?
\end{quote}
\endgroup

Despite the widespread success of Bayesian optimization, this question has been surprisingly overlooked for an extended period and remains unexplored. Moreover, besides Bayesian optimization, this problem widely exists in bandits. Take the classical LinUCB algorithm \citep{abbasi2011improved} as an example, the acquisition function optimization requires an exact optimization solution of context and model parameters, which is highly intractable in practice. And the problem is even worse when it comes to generalized linear bandits and non-linear bandits. Recently, there has been a growing line of research on bandits and global optimization with neural network approximation \citep{zhou2020neural,zhang2020neural,dai2022sample} where the inexact acquisition function optimization problem still exists. In real-world applications, algorithms commonly employ quasi-Newton methods or random search to optimize the acquisition function, resulting in an approximate solution. This practice diverges from the theoretical assumptions of Bayesian optimization, which typically presume that an exact solution $x_t$ is attainable at every iteration $t$. Therefore, we need a systematic study on this problem to understand what theoretical guarantee one can provide when an exact acquisition function solution cannot be obtained. In Table \ref{tab:related}, we summarize existing works that address inexactness in bandit optimization and point out the unique position of our paper in bandit optimization. In the subsection below, we discuss related works.

\subsection{Related Works}\label{sec:rw}
\textbf{Optimization of acquisition functions.}
The optimization of acquisition functions has been a long-neglected issue in BO. \citet{wilson2018maximizing} first studied acquisition function maximization in Bayesian optimization. Since acquisition functions are usually non-convex and high-dimensional, they focused on how to \emph{approximately} optimize acquisition functions. They found that acquisition functions estimated
via Monte Carlo integration are consistently amenable to gradient-based optimization, and EI and UCB can be solved by greedy approaches. Although both their work and this paper study the inexact acquisition function optimization problem, our research distinguishes itself by concentrating on the impact that inexact solutions have on Bayesian optimization.

\begin{table*}[t]
\centering
\begin{tabular}{cccc}
\toprule \noalign{\smallskip}
\textbf{Problem settings}               & Model misspecification & Approximate posterior & Inexact acquisition function \\ \noalign{\smallskip} \hline \noalign{\smallskip}
\textit{Multi-armed bandits}   & \citet{lattimore2020learning}                 & \citet{phan2019thompson}               & \citet{wang2018thompson}                \\
&\citet{foster2020adapting}  & \citet{lu2017ensemble} &  \citet{kong2021hardness};~\citet{perrault2022combinatorial}              \\
\noalign{\smallskip} \hline \noalign{\smallskip}
\textit{Bayesian optimization} & \citet{bogunovic2021misspecified}           & \citet{vakili2021scalable, vakili2022improved}               & {This paper}                     \\ \noalign{\smallskip} \bottomrule \noalign{\smallskip}
\end{tabular}
\caption{Theoretical study of inexactness in bandit optimization
}
\label{tab:related}
\end{table*}

Inexact acquisition function optimization has also been explored in the context of combinatorial semi-bandits, where the decision-making process involves choosing from a combinatorial set of actions to optimize a reward function under constraints. For instance, \citet{xu2021simple} and \citet{ross2013learning} investigate the intricacies of making efficient selections amidst a combinatorial structure of actions by leveraging contextual information to inform the selection of action subsets. These works consider an $\alpha$-approximate oracle \citep{kakade2007playing}, focusing on a relaxed version of regret to account for the inexactness of acquisition optimization. Regret is defined as the difference in reward between the pulled (super-) arm with an approximation of the best reward. \citet{wang2018thompson} show that in general, one cannot achieve no-regret with an approximation oracle in Thompson sampling, even for the classical multi-armed bandit problem. \citet{kong2021hardness} revealed that linear approximation regret for combinatorial Thompson sampling is pathological, while \citet{perrault2022combinatorial} studies a specific condition on the approximation oracle, allowing a reduction to the exact oracle analysis and thus attaining sublinear regret for combinatorial semi-bandits. In contrast to these works concerning multi-armed bandits, our work focuses on Gaussian process bandits and proposes sufficient conditions that allow popular BO algorithms to achieve sublinear regret.

\textbf{Random sampling and discretization methods.}
Random discretization methods have been popular approaches in BO to address the challenges posed by non-convex and high-dimensional acquisition functions. One source of inexactness in acquisition function optimization is due to the discretization inherent in these methods. \citet{bergstra2012random} demonstrate that random search is often more effective than grid search for hyperparameter optimization, particularly as the dimensionality of the problem increases. They show that random search explores a larger, less promising configuration space, but still finds better models within a smaller fraction of the computation time compared to grid search. This approach is beneficial in BO as well, where the curse of dimensionality makes grid-based methods impractical. Our work provides a theoretical understanding of when random sampling works well in the context of inexact acquisition function maximization.
More recently, \citet{gramacy2022triangulation} propose using candidates based on a Delaunay triangulation of the existing input design for Bayesian optimization. These ``tricands'' outperform both numerically optimized acquisitions and random candidate-based alternatives on benchmark synthetic and real simulation experiments. Similarly, \citet{wycoff2024voronoi} introduce the use of candidates lying on the boundary of the Voronoi tessellation of current design points. This approach significantly improves the execution time of multi-start continuous search without a loss in accuracy, by efficiently sampling the Voronoi boundary without explicitly generating the tessellation, thus accommodating large designs in high-dimensional spaces. These methods leverage geometric structures to provide more efficient candidate sets for the acquisition optimization process, aligning with our focus on inexact acquisition function optimization.

\textbf{Misspecified and approximate inference in bandits. 
}
Besides acquisition functions (inner optimization), inexactness also occurs in the (outer) optimization loop of bandit problems, where the objective function class is misspecified with respect to the modeling function class. For instance, in the misspecified linear bandits setting \citep{lattimore2020learning}, the true objective may be nonlinear, while in misspecified Gaussian process bandit optimization \citep{bogunovic2021misspecified}, the objective does not necessarily lie in a function class with a bounded RKHS norm.

\citet{phan2019thompson} studied the effects of approximate inference on the performance of Thompson sampling in $k$-armed bandit problems where only an approximate posterior distribution can be used. With $\alpha$-divergence governing the difference between approximate and true posterior distributions, they proposed a new algorithm that works with sublinear regret in this challenging scenario. In the context of Bayesian optimization, \citet{vakili2021scalable, vakili2022improved} introduced inducing-point-based sparse Gaussian process approximations to address the scalability challenges of full Gaussian process models. In contrast, our work assumes an exact posterior and focuses exclusively on the inexactness introduced by imperfect acquisition function optimization.

\subsection{Our Contributions}
In this paper, we address the problem of inexact acquisition function maximization in Bayesian optimization by analyzing its impact on regret and demonstrating that random grid search is a theoretically justified and practical acquisition solver. We focus on two popular Bayesian optimization algorithms, GP-UCB \citep{srinivas2010gaussian} and GP-TS \citep{chowdhury2017kernelized}, and study their cumulative regret bound when acquisition function maximization problems are not perfectly solved. Our key contributions are summarized as follows.
\begin{itemize}
    \item To the best of our knowledge, our paper is the first to theoretically study the effect of the inexact acquisition function solutions in BO. Despite its widespread practical use, the effect of inexact acquisition maximization has been largely overlooked in the literature. 
    Our work systematically answers the question: ``How do inexact acquisition solutions impact the regret of BO algorithms, and under what conditions do they still guarantee convergence?'' More broadly, 
    our results bridge an important gap in inexact bandit optimization, complementing prior studies on model misspecification and approximate inference in multi-armed bandits and Bayesian optimization (see Table~\ref{tab:related}).
    \item Formally, we introduce a measure of inaccuracy in acquisition solution, worst-case accumulated inaccuracy, and we establish the cumulative regret bounds of both inexact GP-UCB and GP-TS. Our analysis quantifies how the cumulative impact of inexact acquisition solutions influences regret and establishes sufficient conditions under which inexact BO algorithms can still achieve sublinear regrets. These results generalize classical BO regret bounds to more realistic settings where exact acquisition maximization is infeasible.
    \item We theoretically justify random grid search as a valid acquisition solver for BO.
    Our regret bounds show that even with a linearly growing grid size $|\mathcal{X}_t| = \Theta(t)$, random grid search achieves sublinear regret. This significantly relaxes the condition in prior work \citep{chowdhury2017kernelized}, which required an exponentially larger grid size $t^{2d}$, demonstrating that random grid search is both computationally efficient and theoretically grounded. 
    \item 
    We empirically validate the efficiency of random grid search over the existing acquisition function solvers in the context of BO. Our experiments confirm that random grid search not only maintains strong regret performance but also offers substantial computational savings when solving acquisition functions, reinforcing its practical viability as an acquisition function solver.
    \item The Python code to reproduce the numerical experiments in Section \ref{sec:exp} is available at
\url{https://github.com/hwkim12/INEXACT_UCB_GRID}.
\end{itemize}

\section{Preliminaries}\label{sec:pre}
\subsection{Problem Setup and Notations}\label{sec:problem}
We consider a global optimization problem where the goal is to maximize an objective function $f: \cX \rightarrow [0, \infty)$. We use the following notations
\begin{align*}
x^* &= \argmax_{x \in \cX} f(x), \quad f^* =f(x^*),
\end{align*}
where $\cX = [-b, b]^d, ~b \in \mathbb{R}_{>0}$ is a search space of interest, which could be viewed as a set of actions or arms. Unlike the first and second-order optimization methods, where one needs an analytic expression or derivative information of $f$, we allow $f$ to be a blackbox function, whose closed-form expression for the function or the derivative is not necessarily known. Only through evaluations of the function $f$, which could be contaminated by random noise, we aim to identify the maximum of $f$.

To facilitate theoretical understanding of BO methods, we assume the objective function $f$ belongs to the reproducing kernel Hilbert space (RKHS) \citep{aronszajn1950theory,wahba1990spline,berlinet2011reproducing} corresponding to a positive semi-definite kernel function $k: \cX \times \cX \to \mathbb{R}$, denoted by $\mathcal{H}_k$. Following \citep{chowdhury2017kernelized}, we restrict our attention to a set of functions $f$ whose RKHS norm is bounded by some constant $B \in \mathbb{R}_{>0}$. Two important assumptions that will be used throughout the paper are stated below.

\begin{assumption}\label{assump:kernel_less_one}
    The kernel $k(x, x) \le 1$ for all $x \in \cX$.
\end{assumption}

\begin{assumption}\label{assump:kernel_smooth}
    We consider the kernel $k$ to be either a square-exponential kernel or a Mat\'ern kernel with a smoothness parameter $\nu \ge 2$.
\end{assumption}

At each $t^{\text{th}}$ round, we select $x_t$ by maximizing an acquisition function $\alpha_t$ and observe a reward
\begin{align*}
    y_t = f(x_t) + \epsilon_t,
\end{align*}
where the noise $\epsilon_t$ is assumed to be conditionally $R$-sub-Gaussian with respect to the $\sigma$-algebra $\mathcal{F}'_{t-1}$ generated by $\{x_1,\cdots, x_t, \epsilon_1, \cdots, \epsilon_{t-1}\}$ \citep{chowdhury2017kernelized,lattimore2020bandit}, i.e., $\forall \lambda \in \mathbb{R}, ~\mathbb{E}[e^{\lambda\epsilon_t}|\mathcal{F}'_{t-1}] \le \exp(\lambda^2 R^2/2)$ for some $R \ge 0$. 

We assess the performance of a BO algorithm over $T$ iterations based on the cumulative regret $R_T$, given by
\begin{align*}
R_T = \sum_{t=1}^T f^* - f(x_t),
\end{align*}
where $r_t = f^* - f(x_t)$ is referred to as an instantaneous regret at round $t$. The BO algorithm is said to be a zero-regret algorithm if $\lim_{T\rightarrow \infty} R_T/T = 0$, and typical theoretical analyses \citep{srinivas2010gaussian,chowdhury2017kernelized,vakili2021information} aim to show sublinear growth of $R_T$ as the zero-regret algorithm will guarantee the convergence of the algorithm to the maximum. This can be seen from the fact that the simple regret, given by $f^*-\max_{t=1, \cdots, T} f(x_t)$, is bounded above by the average cumulative regret $R_T/T$. Throughout the paper, we use standard big $\mathcal{O}$ notation that hides universal constants, and we use $\tilde{\mathcal{O}}$ to hide all logarithmic factors in $T$.

\subsection{GP-UCB and GP-TS Algorithms}

In the midst of numerous Bayesian optimization methods, two particular strategies of interest are Gaussian Process-Upper Confidence Bound (GP-UCB) introduced in \citet{srinivas2010gaussian} and Gaussian Process-Thompson Sampling (GP-TS) proposed by \citet{chowdhury2017kernelized}. More formally, to design BO algorithms, one typically imposes a zero-centered GP prior to the target objective function $f$ and models the random noise through a Gaussian random variable with variance $\tau$. 
Conditioning on all $t-1$ observations prior to obtaining $t$ th observation, the posterior mean $\mu_{t-1}$ and standard variance $\sigma^2_{t-1}(x)$ of the Gaussian process are given by
\begin{align*}
    \mu_{t-1}(x) &= k_{t-1}(x)^T (K_{t-1} + \tau I)^{-1} Y_{t-1}, \\
    \sigma^2_{t-1}(x) &= k(x,x) - k_{t-1}(x)^T (K_{t-1} + \tau I)^{-1} k_{t-1}(x),
\end{align*}
where
\begin{align*}
    k_{t-1}(x) &= [k(x_1, x), \cdots, k(x_{t-1}, x)]^T,\\
    Y_{t-1} &= [y_1, \cdots, y_{t-1}],\\
    K_{t-1} &= [k(x, x')]_{x, x' \in \{x_1, \cdots, x_{t-1}\}}.
\end{align*}
In other words, the posterior distribution of $f$ conditioning on $\{(x_i, y_i)\}_{i=1}^{t-1}$ is given by the Gaussian process with mean function $\mu_{t-1}$ and posterior variance $\sigma^2_{t-1}$.

\begin{algorithm}[!htbp]
\caption{GP-UCB \citep{srinivas2010gaussian, chowdhury2017kernelized} \label{alg:gpucb}}
\begin{algorithmic}[1]
\STATE {\bf Input}: Kernel $k;$ Total number of iterations $T;$ Initial design points $X_0;$ Initial observations $Y_0$.
    \STATE Construct $\mu_{0}(x)$ and $\sigma_{0}(x)$ using $X_0, Y_0$.
    \FOR{$t = 1, \ldots, T$}
    \STATE $\beta_t\leftarrow B+R\sqrt{2(\gamma_{t-1}+1 +\log(1/\delta))}$.
    \STATE 
$x_t\leftarrow \argmax_{x}
\mu_{t-1}(x)+\beta_t\sigma_{t-1}(x)$.
    \STATE  Observe $y_t = f(x_t) + \epsilon_t$.
    \STATE  $X_t = X_{t-1} \cup \{
x_t\}, Y_t = Y_{t-1} \cup \{
y_t\}$.
    \STATE  Update $\mu_{t}(x)$ and $\sigma_{t}(x)$ using $X_t, Y_t$.
    \ENDFOR
\end{algorithmic}
\end{algorithm}
\begin{algorithm}[!htbp]
\caption{GP-TS \citep{chowdhury2017kernelized} \label{alg:gpts}}
\begin{algorithmic}[1]
\STATE {\bf Input}: Kernel $k;$ Total number of iterations $T;$ Initial design points $X_0;$ Initial observations $Y_0$.
    \STATE Construct $\mu_{0}(x)$ and $\sigma_{0}(x)$ using $X_0, Y_0$.
    \FOR
    {$t = 1, \ldots, T$}
    \STATE $\beta_t\leftarrow B+R\sqrt{2(\gamma_{t-1}+1 +\log(2/\delta))}$.
    \STATE Sample $f_t(\cdot)\sim\mathcal{GP}_D\left(\mu_{t-1}(\cdot), s_t^2 \right)$, with $s_t^2(\cdot) = \beta^2_t \sigma^2_{t-1}(\cdot)$  .
    \STATE Choose the current decision set $\mathcal{X}_t \subset \mathcal{X}$ of size $|\mathcal{X}_t| = (2BLbdt^2)^d$.
    \STATE  Set
$x_t = \arg\max_{x \in \mathcal{X}_t}  f_t(x)$.
    \STATE  Observe $y_t = f(x_t) + \epsilon_t$.
    \STATE  $X_t = X_{t-1} \cup \{x_t\}, Y_t = Y_{t-1} \cup \{y_t\}$.
    \STATE Update $\mu_{t}(x)$ and $\sigma_{t}(x)$ using $X_t, Y_t$.
    \ENDFOR
\end{algorithmic}
\end{algorithm}

For the choice of an acquisition function to facilitate a Bayesian optimization strategy, \citet{srinivas2010gaussian} considered a UCB function of the form
$$
\alpha_t(x) \coloneqq \mu_{t-1}(x) + \beta_t\sigma_{t-1}(x),
$$
where $\beta_t$ is chosen to balance exploitation (picking points with high function values) and exploration (picking points where the prediction based on the posterior distribution is highly uncertain). In their work, they termed the Bayesian optimization strategy based on the UCB acquisition function as GP-UCB, which has since become popular in both theoretical analyses and empirical applications. In the meantime, \citet{chowdhury2017kernelized} proposed another BO strategy under the name of GP-TS, which leverages the acquisition function of the form
$$
\alpha_t(x) \coloneqq f_t(x), 
$$
where $f_t$ is a sample path from the posterior Gaussian process with the mean function $\mu_{t-1}$ and variance function $\beta_t^2 \sigma^2_{t-1}$. Such a strategy has been widely used under the name of Thompson sampling in BO and bandit literature \citep{agrawal2013thompson,russo2014learning,kandasamy2015high}. In practice, optimizing a randomly drawn sample path $f_t$ can be just as hard as finding the optimum of the target objective $f$. In reflection of such difficulty, a more practical version of the aforementioned TS-based BO strategy through discretizing the search space for the acquisition function was proposed in \citet{chowdhury2017kernelized}.

For the theoretical analysis, \citet{srinivas2010gaussian} and \citet{chowdhury2017kernelized} considered an increasing sequence for the choice of $\beta_t$, and in particular, \citet{chowdhury2017kernelized} set $\beta_t$ to grow at a rate of $\mathcal{O}(\sqrt{\gamma_T})$ where $\gamma_T$ is a kernel-dependent quantity known as the maximum information gain at time $t$, given by
$$
\gamma_T = \max_{\cX_T \subset \cX: |\cX_T| = T} \frac{1}{2}\log\det(I_T + \tau^{-1}K_T).
$$
The formal procedures for Bayesian optimization using the UCB and Thompson sampling acquisition functions are outlined in Algorithms \ref{alg:gpucb} and \ref{alg:gpts}, respectively. Furthermore, the state-of-the-art growth rate of the maximum information gain is provided in \citet{vakili2021information}, which states that $\gamma_T = \mathcal{O}\left(T^{\frac{d}{d+2\nu}} \log^{\frac{2\nu}{2\nu+d}}(T)\right)$ for Mat\'ern Kernel and $\gamma_T = \mathcal{O}\left(\log^{d+1}(T)\right)$ for SE Kernel. Although we work with UCB and TS in this paper, our analysis can be extended to more acquisition functions\footnote{\citet{wang2017max} showed that MES, MVES, and PI can be cast as special cases of UCB with appropriately chosen confidence parameters, and thus fall within the scope of our theoretical framework.}.

\section{Regret bounds under inexact acquisition function maximization
}\label{sec:inexact}

Although theoretical developments in BO often assume that the acquisition maximization is solved perfectly, this is rarely the case in practice. For example, when optimizing the UCB acquisition function—whose gradient information is typically available—a popular approach is to use quasi-Newton methods with multiple starting points \citep{frazier2018tutorial, gramacy2020surrogates, pourmohamad2021bayesian}. While these methods have proven effective for convex problems, the UCB acquisition function is usually nonconvex. Consequently, factors such as the number of starting points, the total number of optimization iterations, and the stopping criteria all impact the quality of the acquisition solution obtained.

In TS-based BO methods, it is common to use a discretized subset of the action space because gradients for the sample paths are difficult to obtain. In these cases, a sorting algorithm is used to select the maximum value from the posterior sample path within the chosen grid \citep{kandasamy2018parallelised, chowdhury2017kernelized}. Although theoretical studies have shown that the TS strategy converges to the optimal function value, the grid size must grow exponentially with the number of dimensions \citep{chowdhury2017kernelized}. Moreover, the granularity of the discretization heavily influences the computational cost. As a result, typical implementations use the finest discretization possible within the available computational budget, which can lead to inexact acquisition solutions.

In this section, we examine how a series of inexact acquisition function maximizations affects BO strategies. Although we restrict our attention to acquisition functions of GP-UCB and GP-TS, our analysis is not pertained to any specific acquisition function solver. We introduce a measure to quantify the accumulated inaccuracies caused by these inexact maximizations, with the goal of understanding how they impact the existing cumulative regret bounds. 

\subsection{Accumulated inaccuracy}
We first introduce a measure of inaccuracies that arise when solving acquisition function optimization problems. Let 
$
\alpha_t^* = \max \alpha_t(x),
$
be the maximum value of the acquisition function at $t^{\text{th}}$ iteration, which is at least as large as $\alpha_t(x_t)$
where $x_t$ is the action selected at that iteration. For our analysis, we assume that the acquisition functions 
$\{\alpha_t\}_{t=1}^T$ are nonnegative and always achieve a strictly positive maximum. This assumption can be easily met by adding a positive constant to the acquisition function at each iteration or by appropriately restricting the search space. A detailed justification for the existence of such a constant is provided in Appendix~\ref{sec:app:gpucb}. We emphasize that this modification does not change the chosen action, nor does it incur any additional computational cost.

At each iteration, we quantify the accuracy of the acquisition optimization solution by the ratio of the acquisition function value at the selected action to its maximum value:
$$
\eta_t \coloneqq \frac{\alpha_t(x_t)}{\alpha_t^*} \in \left[\tilde{\eta_t},1\right],
$$
so that $\eta_t = 1$ if the acquisition function is maximized perfectly. In other words, a larger value of $\eta_t$ indicates a more accurate solution to the $t^{\text{th}}$ acquisition optimization problem. Here, $\tilde{\eta_t}$ represents the worst-case accuracy of the $t^{\text{th}}$ acquisition function solver. The overall impact of inaccurate solutions across iterations is then captured by the worst-case accumulated inaccuracy,
$$
M_T = \sum_{t=1}^T (1-\tilde \eta_t) \in [0, T],
$$
which represents the total inaccuracy permitted over 
$T$ rounds of Bayesian optimization. Using this measure of accumulated inaccuracy, we establish cumulative regret bounds for GP-UCB and GP-TS in the remainder of this section, in contrast to existing theoretical works that assume perfect acquisition solutions.

\subsection{Regret Bounds}\label{sec:gpucb}

Recall that for the kernel function under consideration, we have $|k(x, x')| \le 1$. Furthermore, the $t^{\text{th}}$ action $x_t$ satisfies $\alpha^*_t \coloneqq \max \alpha_t(x) \ge  \alpha_t(x_t) \ge \eta_t \alpha^*_t$, where $\alpha_t(x) = \mu_{t-1}(x) + \beta_t \sigma_{t-1}(x)$. We denote by $\cH_k$ the RKHS associated with a chosen kernel $k$. We then establish cumulative regret bounds for the GP-UCB algorithm in the presence of inexact acquisition function solutions.

\begin{theorem}\label{thm:GP_UCB_MFC}
Under assumptions \ref{assump:kernel_less_one} and \ref{assump:kernel_smooth}, suppose the objective function $f \in \mathcal{H}_k$ satisfies $\|f\|_{\mathcal{H}_k} \le B$. Then, the inexact GP-UCB algorithm with $\beta_t = B + R\sqrt{2(\gamma_{t-1} + 1 + \log(1/\delta))}$ and worst-case accumulated inaccuracy $M_T$ achieves a cumulative regret bound of the form:
$$
R_T = \mathcal{O}\left(\gamma_T\sqrt{T}  + M_T\sqrt{\gamma_T} \right),
$$ 
with probability $1-\delta$.
\end{theorem}

Note that when $M_T=0$ (exact maximization), the cumulative bound above recovers standard regret guarantees for GP-UCB \citep{srinivas2010gaussian, chowdhury2017kernelized}. The additive factor of $M_T\sqrt{\gamma_T}$ accounts for the effect of inaccurate acquisition solutions. This implies that if the worst-case accumulated inaccuracy $M_T$ and the maximum information gain $\gamma_T$ do not grow too quickly so that $\frac{M_T\sqrt{\gamma_T}}{T} \to 0$ as $T\to\infty$, the inexact GP-UCB algorithm will converge to the optimal solution. For example, with a squared-exponential kernel, it has been shown that $\gamma_T$ grows logarithmically \citep{vakili2021information}. Theorem \ref{thm:GP_UCB_MFC} thus indicates that the inexact GP-UCB algorithm converges asymptotically to the optimal function value, provided that the worst-case accumulated inaccuracy $M_T$ is sublinear.   

One can also establish a similar result for the GP-TS algorithm, as stated below. In contrast to the GP-TS algorithm introduced in \citep{chowdhury2017kernelized}, our formulation accounts for the extra uncertainty induced by inexact solutions through the use of sample paths obtained from the GP posterior with an enlarged variance.

\begin{theorem}\label{thm:GP_TS_MFC}
Under assumptions \ref{assump:kernel_less_one} and \ref{assump:kernel_smooth}, suppose the objective function $f \in \mathcal{H}_k$ satisfies $\|f\|_{\mathcal{H}_k} \le B$. Then, the inexact GP-TS algorithm with variance $s^2_{t-1}(x) = \left(\beta_t \sigma_{t-1}(\cdot) + \tilde v_t\right)^2$, $\tilde v_t = (\frac{1}{\tilde \eta_t}-1)B$ and worst-case accumulated inaccuracy $M_T$ achieves a cumulative regret bound of the form:
$$
R_T = \mathcal{O}\left(\gamma_T\sqrt{T \log T} + M_T \sqrt{\gamma_T\log T}\right),
$$ 
with probability $1-\delta$. 
\end{theorem}

The above regret bound matches the existing bound for GP-TS \citep{chowdhury2017kernelized}, up to an additive factor of $M_T\sqrt{\gamma_T\log T}$, which captures the inaccuracies in the acquisition function solutions. Similar to the GP-UCB case, if the worst-case accumulated inaccuracy $M_T$ and the maximum information gain $\gamma_T$ do not increase too rapidly so that $M_T\sqrt{\gamma_T \log T} /T \to 0$ as $T\to\infty$, Theorem \ref{thm:GP_TS_MFC} shows that the inexact GP-TS algorithm will achieve a sublinear cumulative regret bound.

An overall implication of the theorems established in this section is that BO strategies can retain asymptotic convergence guarantees even without exact acquisition function solutions, provided that the accuracy of these solutions improves over time. This insight naturally motivates the exploration of BO strategies that are computationally more efficient by reducing the effort spent on acquisition function maximizations while preserving convergence guarantees.

\section{Solving acquisition function through random grid search}\label{sec:random_grid}

Building on the insights from Section \ref{sec:inexact}, we further investigate a simple yet computationally efficient approach: a random grid search for acquisition function maximization. In this method, to maximize an acquisition function $\alpha_t$, one obtains a set of random points from the search space and selects the one with the highest acquisition function value. This approach has been employed in numerous BO problems and has demonstrated its effectiveness \citep{kandasamy2018parallelised, pourmohamad2021bayesian}. Importantly, our results provide the first theoretical validation for this approach, showing that even with a linear growth in grid size, one can still achieve sublinear regret.

Concretely, we analyze the GP-UCB and GP-TS algorithms using a random grid search for acquisition maximization and derive their cumulative regret bounds. As noted in the previous section, the accuracy of the acquisition solver must improve over iterations. To achieve this, we employ a sequence of random grids $\{\cX_t\}_{t \in \mathbb{N}}$ whose size grows linearly with the iteration count, i.e., $|\cX_t| = \Theta(t)$. Such a strategy of increasing the grid size has proven empirically successful \citep{kandasamy2018parallelised} in combination with the TS acquisition function.

To precisely define our acquisition function optimization procedure, consider a grid of random samples $\cX_t \subset \cX$ that serves as the search space for the $t^{\text{th}}$ acquisition function optimization. At each $t^{\text{th}}$ iteration, we set
\begin{align*}
x_t \coloneqq \begin{cases}
    \argmax_{x \in \mathcal{X}_t} \mu_{t-1}(x) + \beta_t \sigma_{t-1}(x) \quad \text{for GP-UCB} \\
    \argmax_{x \in \mathcal{X}_t} f_{t}(x) \quad \text{for GP-TS}.
\end{cases}
\end{align*}
We refer to the first strategy as random grid GP-UCB and the second as random grid GP-TS. For the choice of random grid, we make the following assumption and state our cumulative regret bound results.

\begin{assumption}\label{assum:grid}
At iteration $t$, $\mathcal{X}_t$ consists of $t$ independent samples drawn uniformly from $\cX$. Additionally, the sequence of random grids $\{\cX_t\}_{t \in \mathbb{N}}$ is independent across iterations.
\end{assumption}

\begin{remark}
Our results extend to any sampling scheme whose density remains strictly positive over the search space. Here, we focus on a uniform random grid because it is the easiest to implement in practice.
\end{remark}

\begin{theorem}\label{thm:random_UCB}
Under assumptions \ref{assump:kernel_less_one}, \ref{assump:kernel_smooth} and \ref{assum:grid}, suppose $f \in \mathcal{H}_k$ with $\|f\|_{\mathcal{H}_k} \le B$. With probability $1-\delta$, the random grid GP-UCB with $\beta_t = B + R\sqrt{2(\gamma_{t-1} + 1 + \log(2/\delta))}$ yields 
\begin{align*}
R_T = \mathcal{O}\left(\gamma_T \sqrt{T}\right) + \tilde{\mathcal{O}}\left(T^{\frac{d-1}{d}}\right).
\end{align*}
\end{theorem}

\begin{theorem}\label{thm:random_TS}
Under assumptions \ref{assump:kernel_less_one}, \ref{assump:kernel_smooth} and \ref{assum:grid}, suppose $f \in \mathcal{H}_k$ with $\|f\|_{\mathcal{H}_k} \le B$. With probability $1-\delta$, the random grid GP-TS with $\beta_t = B + R\sqrt{2(\gamma_{t-1} + 1 + \log(3/\delta))}$ yields 
\begin{align*}
R_T = \mathcal{O}\left(\gamma_T \sqrt{T \log T} \right) + \tilde{\mathcal{O}}\left(T^{\frac{d-1}{d}}\right).
\end{align*} 

\end{theorem}
\begin{remark}
Similar to the regret bounds from Section \ref{sec:inexact}, our cumulative regret bounds now incorporate an additional factor that accounts for the inaccuracies introduced by the random grid search. In both Theorem \ref{thm:random_UCB} and Theorem \ref{thm:random_TS}, these effects are represented by the term $\tilde{\mathcal{O}}\left(T^{\frac{d-1}{d}}\right)$. Importantly, our results demonstrate that even a simple random grid search is sufficient for acquisition function maximization, ensuring that Bayesian optimization asymptotically converges to the global optimum under a suitable growth rate of the maximum information gain $\gamma_T$.
\end{remark}

\begin{remark}
As the grid size increases superlinearly, the order of the inaccuracy factor can potentially decrease. For instance, if the grid size grows at an order of $t^2$, we would instead get the inaccuracy factor of $\tilde{\mathcal{O}}\left(t^{\frac{d-2}{d}}\right)$ or if the grid size grows at an order of $t^d$, we would approximately get the inaccuracy factor of $\mathcal{O}(\log t)$.
\end{remark}

\begin{remark}
In our analyses, we mainly considered the squared-exponential and Mat\'ern kernels, thanks to their popularity and existing growth rate results on their maximum information gain. The key aspect of the maximum information gain growth rate can be characterized by the decaying rate of eigenvalues associated with the kernel \citep{vakili2021information}. Naturally, as long as the decaying rate of a kernel is known, our analyses can also be extended to other types of kernels. Furthermore, our analysis can also be extended to settings with non-stationary noise variance \citep{iwazaki2025improved}.
\end{remark}

Our theoretical results indicate that random grid search is sufficient for optimizing UCB and TS acquisition functions to achieve sublinear regret bounds for both GP-UCB and GP-TS. Although many successful implementations of random grid search for acquisition function maximization exist, there has been little theoretical justification for this approach. The closest work we are aware of is by \citep{chowdhury2017kernelized}, which incorporated a grid search within the GP-TS algorithm to solve the acquisition function optimization problem. Their analysis yields sublinear regret for certain kernels but requires the grid size to scale as $t^{2d}$, where $t$ is the iteration index and $d$ is the problem dimension. This requirement is computationally demanding; indeed, in the same paper, the authors used a fixed grid size for numerical experiments, highlighting the gap between theoretical developments and practical implementations. Unlike their results, our regret bounds justify the empirical success of the TS acquisition function when using a random grid search that grows linearly with the iteration number, as demonstrated in \citep{kandasamy2018parallelised}. This result not only demonstrates the theoretical soundness of using random grid search for acquisition maximization but also provides rigorous justification for the empirical findings in \citep{kandasamy2018parallelised}, which observed that a random grid search of order $\mathcal{O}(t)$ is robust and efficient for TS algorithms. Furthermore, as Section~\ref{sec:exp} will show, the computational efficiency of random grid search approach offers a significant practical advantage over more complex solvers such as quasi-Newton methods, without sacrificing convergence guarantees.

\section{Experiments}\label{sec:exp}
In this section, we first present experimental results on synthetic functions and then a real-world automated machine learning task. In Appendix \ref{app:exp}, we provide additional details of the experimental setup and more ablation studies.

\subsection{Synthetic Experiments}
\begin{figure*}[!h]
    \centering
      \includegraphics[width=.9\linewidth]{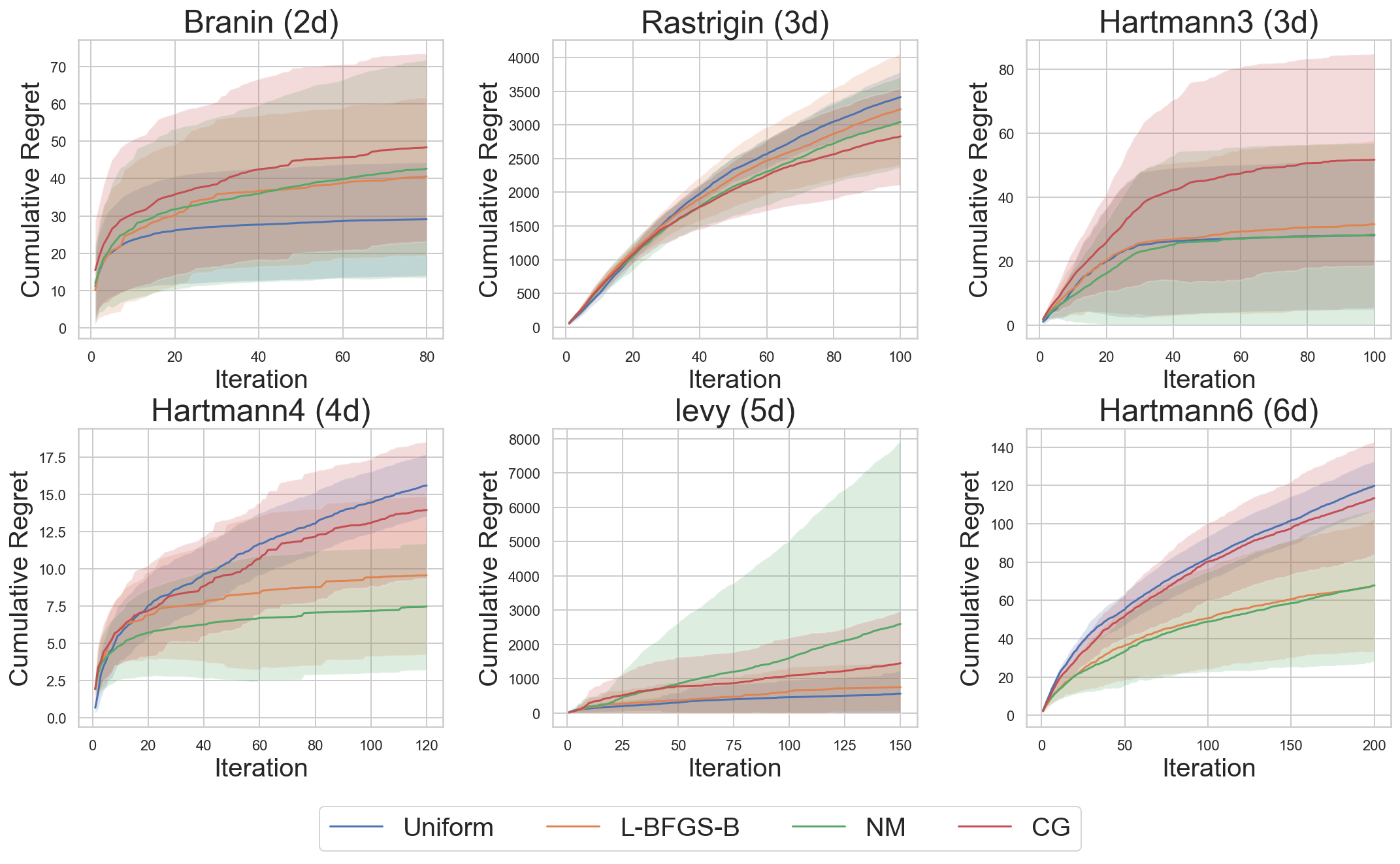}
    \caption{Cumulative regret comparison between acquisition function solvers.}\label{fig:acq_cum}
\end{figure*}
We conduct numerical experiments on six benchmark functions to demonstrate the effectiveness of random grid search combined with the UCB acquisition function. We refer the reader to \citep{kandasamy2018parallelised} for a demonstration of the effectiveness of random grid search when using the TS acquisition function. The test functions include Branin (2D on $[-5, 10] \times [0, 15]$), Rastrigin (3D on $[-5.12, 5.12]^3$), Hartmann3 (3D on $[0,1]^3$), Hartmann4 (4D on $[0,1]^4$), Levy (5D on $[-10, 10]^5$), and Hartmann6 (6D on $[0,1]^6$). For each problem, a set of initial design points is generated using a Sobol sequence. These points are scaled to the appropriate domain. We vary the number of initial design points by problem (20 for Branin, 30 for Rastrigin and Hartmann3, 40 for Hartmann4, 50 for Levy, and 60 for Hartmann6), reflecting the dimension of the objective function. After initialization, BO algorithms with different choices of acquisition function solvers are performed for a predetermined number of iterations: 80 iterations for Branin, 100 for Rastrigin, Hartmann3, and Hartmann4, 150 for Levy, and 200 for Hartmann6. In each iteration, a Gaussian process with a Mat\'ern kernel is fitted to the available data, and the UCB acquisition function is optimized with an exploration parameter $\beta_t = \sqrt{\log(t+2)}$ for each iteration $t$. The optimization of this acquisition function is carried out using one of four acquisition optimization methods: a uniform random grid search (abbreviated as Uniform), popular quasi-Newton methods including Limited-memory Broyden–Fletcher–Goldfarb–Shanno (abbreviated as L-BFGS-B) as well as Nelder–Mead (abbreviated as NM), and the conjugate gradient (abbreviated as CG) method. For the random grid size $|\cX_t|$, we set it to be $100t$, which scales linearly in terms of the number of iterations. Each method is evaluated over 20 independent experiments (with varying random seeds) to provide a statistically valid comparison in terms of both cumulative regret and computational time.

\begin{figure}[!htbp]
  \centering
  \includegraphics[width=\linewidth]{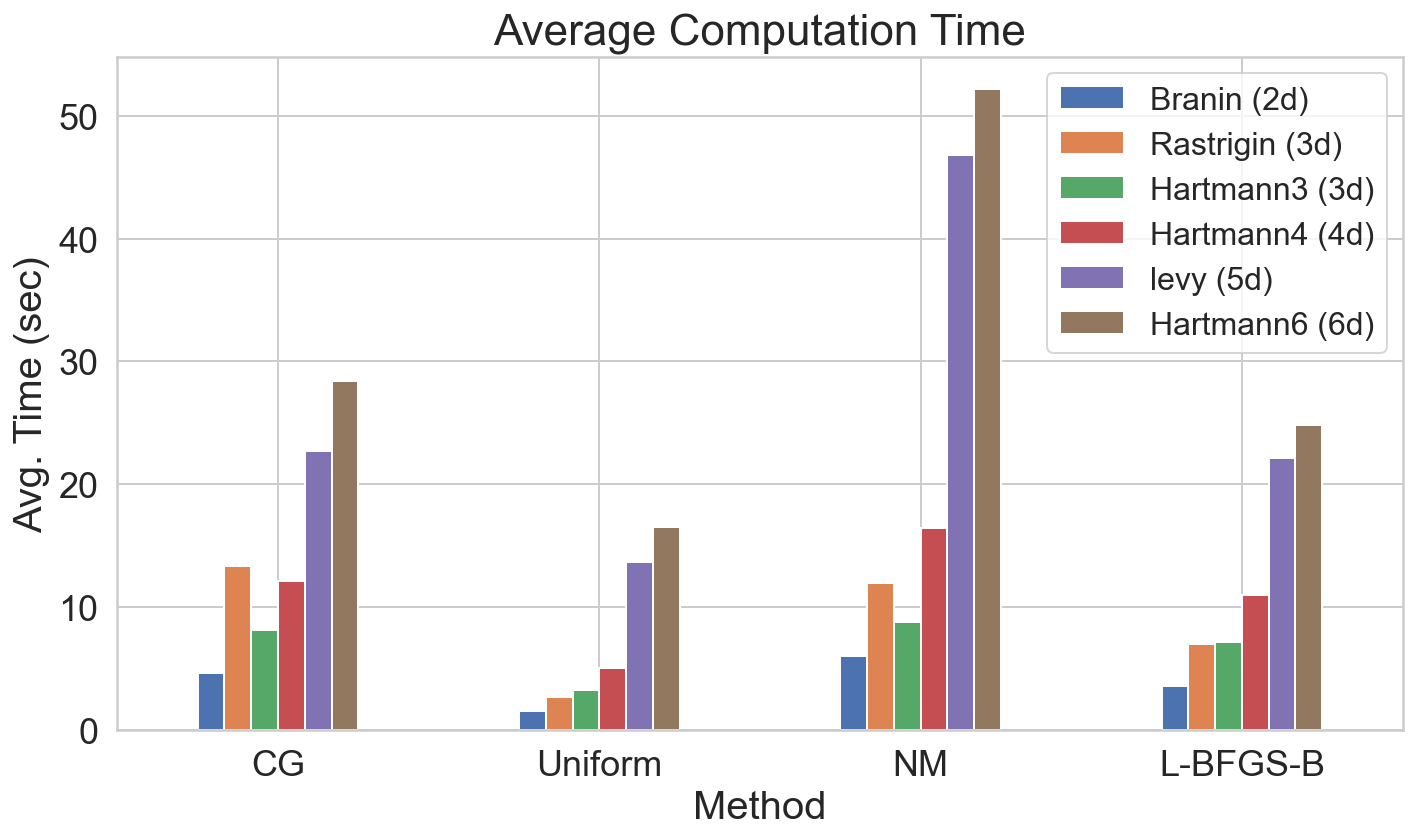}
  \caption{Computational time comparison between acquisition function solvers.}\label{fig:acq_comp_time}
\end{figure}

Figure \ref{fig:acq_cum} shows that, across a range of dimensions and objective landscapes, the uniform sample grid search consistently achieves competitive cumulative regret and keeps pace with, or in some cases outperforms, the popular acquisition optimization tools. For instance, in the 2D Branin problem, the uniform sample grid search’s cumulative regret curve is superior to that of other sophisticated optimization algorithms over the entire horizon, demonstrating rapid improvement from the outset. On the more challenging 3D Rastrigin, 3D Hartmann3, and 4D Hartmann4 functions, uniform sample grid search continues to exhibit a steady reduction in regret, closely matching or surpassing other solvers. Even for the higher-dimensional Levy (5D) and Hartmann6 (6D) problems, the uniform sample grid search approach remains impressively competitive, achieving a final cumulative regret that is comparable to the gradient-based methods.

Figure \ref{fig:acq_comp_time} further highlights the uniform sample grid search’s practical advantages: it maintains one of the lowest average runtimes across all problems, standing in stark contrast to NM, which exhibits significantly longer runtimes (especially in the 5D and 6D settings)\footnote{We provide a brief comparison of the computational complexity of the acquisition optimizers used in Figure~\ref{fig:acq_comp_time}. For the random grid search, the grid size scales as $\mathcal{O}(t)$, roughly yielding a total complexity of $\tilde{\mathcal{O}}(T^2)$ over $T$ iterations. In contrast, gradient-based methods (e.g., L-BFGS-B, CG) require iterative optimization with per-iteration costs dependent on the dimensionality. Let $I_t$ denote the number of gradient updates at step $t$, and let $\text{cost}(\text{gradient})$ be the cost of a single gradient evaluation (e.g., $\mathcal{O}(d^3)$ for Newton, $\mathcal{O}(d^2)$ for CG/BFGS). Then the overall complexity becomes $\sum_{t=1}^T \mathcal{O}(I_t \times \text{cost}(\text{gradient}))$.}.

While the gap in runtime between uniform random grid search and existing acquisition function solvers narrows in higher dimensions—partly due to the increased number of BO iterations and corresponding grid evaluations—the uniform sample grid search continues to offer a favorable balance between efficiency and regret reduction, making it an appealing choice when balancing rapid progress in regret reduction with efficient use of computational resources.

\subsection{Real-World AutoML Task}
Similarly, we perform a hyperparameter tuning task focusing on improving the validation accuracy of the gradient boosting model (from the Python \texttt{sklearn} package) on the breast-cancer dataset (obtained from the UCI machine learning repository). Details of 11 hyperparameters are shown in Appendix \ref{app:exp}. Figure \ref{fig:automl} shows the performances using different inner optimization methods for GP-UCB. We can see that random grid search (``uniform'' in the figure) achieves relatively low cumulative regret as well as low computation time from Figure \ref{fig:automl}.

\begin{figure}[t]
    \centering
    \begin{minipage}{0.9\linewidth}\centering
		\includegraphics[width=\textwidth]{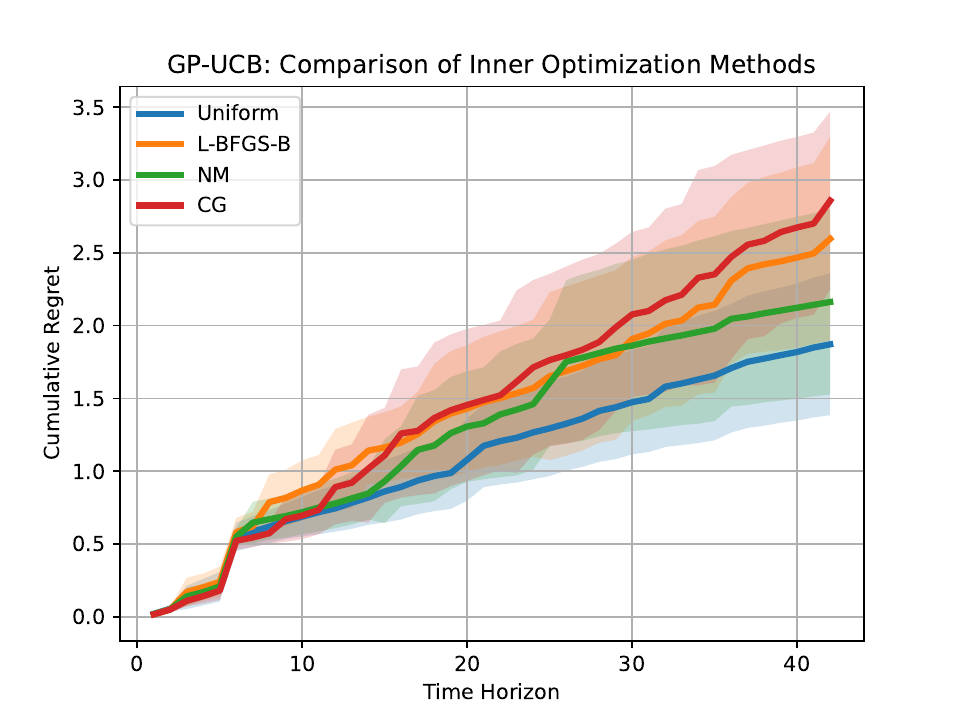}
	\end{minipage}
    \begin{minipage}{0.8\linewidth}\centering
		\includegraphics[width=\textwidth]{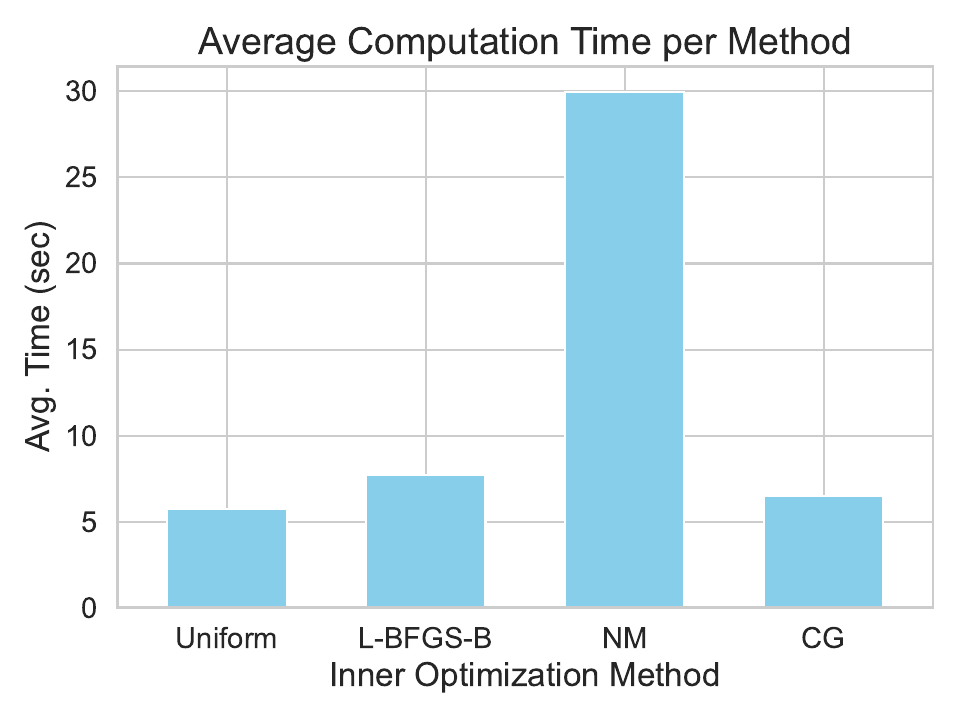}
	\end{minipage}
    \caption{Cumulative regret (upper) and average computational time (lower) of different acquisition optimization tools on the real-world AutoML task.}
    \label{fig:automl}
\end{figure}

\section{Conclusion}\label{sec:con}
In this paper, to the best of our knowledge, we study the inexact acquisition function maximization problem for BO for the first time, a topic that has been largely overlooked. Existing works predominantly operate under the assumption that an exact solution is attainable, thus allowing for the establishment of rigorous theoretical upper bounds. However, in practice, finding exact solutions is difficult due to the multimodal nature of acquisition functions

To address this discrepancy, we define a measure of inaccuracy in acquisition solution, referred to as the worst-case accumulated inaccuracy. We establish cumulative regret bounds for both inexact GP-UCB and GP-TS. Our bounds show that under some conditions on the worst-case accumulated inaccuracy, inexact BO algorithms can still achieve sublinear regrets. Moreover, we extend to provide the first theoretical validation for random grid search, showing that even with a linear growth in grid size, one can still achieve sublinear regret. Our experimental results also validate the effectiveness of random grid search in solving acquisition maximization problems.

Although our work focuses on GP-UCB and GP-TS, our analysis can be extended to other acquisition functions with similar structural properties. Additionally, our framework opens avenues for studying adaptive inexact optimization strategies, where the computational effort allocated to acquisition function maximization is dynamically adjusted based on accumulated inaccuracy. We hope our study will spur systematic theoretical work on inexact acquisition optimization and inspire empirical investigations into efficiently allocating computational resources in BO.

\begin{acknowledgements} 
This work was partially done when HK and CL were at the University of Chicago. This work was supported in part by the National Science Foundation under Grant No. IIS 2313131, IIS 2332475, and CMMI 2037026. 
The authors would like to thank the reviewers and the AC for helpful comments that improved the final version of this paper.
\end{acknowledgements}

\bibliography{bib}

\newpage
\onecolumn
\appendix

\section{Proof for Theoretical Statements}
\subsection{Assumptions and Key Lemmas}
We first list out necessary assumptions and key lemmas to establish theorems stated in the manuscript. 
\begin{assumption}\label{append_assump:kernel_less_one}
    The kernel $k(x, x) \le 1$ for all $x \in \cX$.
\end{assumption}

\begin{assumption}\label{append_assump:kernel_smooth}
    We consider the kernel $k$ to be either a square-exponential kernel or a Mat\'ern kernel with smoothness parameter $\nu \ge 2$.
\end{assumption}

\begin{assumption}\label{append_assum:grid}
At iteration $t$, $\mathcal{X}_t$ is a collection of $t$ random samples drawn uniformly from $\cX$. Furthermore, we assume that the set of uniform random grids $\{\cX_t\}_{t \in \mathbb{N}}$ are independent. 
\end{assumption}

\begin{lemma}[Theorem 2 of \citet{chowdhury2017kernelized}]\label{lem:conf} Suppose $f \in \mathcal{H}_k$ with $\| f\|_{\mathcal{H}_k} \le B$. Then the following statement holds with probability at least $1-\delta$, for all $x \in \mathcal{X}$ and $t \in \mathbb{N}$,
\begin{align*}
|f(x) - \mu_{t-1}(x)| \leq \beta_t \sigma_{t-1}(x),
\end{align*}
where $\beta_t = B + R\sqrt{2(\gamma_{t-1} + 1 + \log(1/\delta))}$ and $\gamma_{t-1}$ is the maximum information gain after $(t-1)$ rounds.
\end{lemma}

\begin{lemma}[Lemma 4 of \citet{chowdhury2017kernelized}]\label{lem:sigma_bound} Let $\{z_1, \cdots, z_T\}$ be the points selected by arbitrary BO strategy. Then the following holds, 
\begin{align*}
\sum_{t=1}^T \sigma_{t-1}(z_t) \leq \sqrt{4(T+2)\gamma_T}
\end{align*}
where $\gamma_T$ is the maximum information gain after $T$ rounds. 
\end{lemma}

Throughout the remainder of the appendix, we use the notation $\mathcal{O}(t^\xi)$, for arbitrarily small $\xi > 0$, to represent logarithmic factors in $t$.

\begin{lemma}[Proposition 4 of \citet{helin2022introduction}]\label{lem:gap}
Let $\tilde x_1,...,\tilde x_t$ be $t$ uniform random samples from a hyperrectangle $\cX \subseteq \mathbb{R}^d$. Define $h_t = \sup_{x \in \cX} \inf_{i=1,...t} \|x - \tilde x_i\|$, then there exists $t^*$ such that $\forall t \ge t^*$, 
\begin{align*}
\E [h_t] = \mathcal{O}(t^{-\frac{1}{d}+\xi}),
\end{align*}
where $\xi > 0$ is an arbitrarily small positive constant.
\end{lemma}

In the following lemma, we establish a high-probability upper bound on the cumulative discrepancy between the random grid $\cX_t$ and the search space $\cX$.
\begin{lemma}\label{lemma:grid_disc}
    Let $h_t = \sup_{x \in \mathcal{X}}\inf_{\tilde x_i \in \mathcal{X}_t} \|x-\tilde x_i\|$. For $\delta > 0$ small, with probability at least $1-\delta$, we have
    $$
        \sum_{t=1}^T h_t \le \sum_{t=1}^{t^*} h_t +  \frac{CT^{\frac{d-1}{d} + \xi}}{\delta} ,
    $$
    for some $C > 0$ and $t^* \in \mathbb{N}$, with arbitrarily small $\xi > 0$.
\end{lemma}
\begin{proof}
    From Lemma \ref{lem:gap}, we know that there exists $C > 0, \xi > 0$ and $t^* \in \mathbb{N}$ such that $\mathbb{E}[h_t] \le C t^{-\frac{1}{d} + \xi}$ for all $t \ge t^*$. From the Markov's inequality, we know that
    $$
     \mathbb{P}\left[\sum_{t=t^*}^T h_t > \frac{CT^{\frac{d-1}{d} + \xi}}{\delta}\right] \le \frac{\mathbb{E}\left[\sum_{t=t^*}^T h_t \right] }{CT^{\frac{d-1}{d} + \xi}} \delta \le \delta
    $$
    where the second inequality follows from $\mathbb{E}[\sum_{t=t^*}^T h_t]  \le CT^{\frac{d-1}{d}+\xi} $.
\end{proof}

\subsection{GP-UCB/GP-TS with inexactness}\label{sec:app:gpucb}
We first provide justifications for the nonnegativity of the UCB acquisition function after the constant shift. To this end, we show the existence of a constant $C$ such that $\alpha_t + C$ is non-negative on the search space $\cX$.

\begin{lemma}
    For some $\delta > 0$, a constant shifted $t^{\text{th}}$ acquisition function of the GP-UCB algorithm, given by
    $$
        \alpha_t(x) + C = \mu_{t-1}(x) + \beta_t \sigma_{t-1}(x) + C
    $$
    is non-negative on $\cX$ with probability $1-\delta$, for some constant sequence $C \in \mathbb{R}$. 
    \end{lemma}
\begin{proof}
From Lemma \ref{lem:conf}, we know that
$$
f(x) \le \mu_{t-1}(x) + \beta_t \sigma_{t-1}(x),
$$
for all $x \in \cX$ with probability $1-\delta$. Therefore, 
$$
0 \le f(x) + |f(x)|  \le  \mu_{t-1}(x) + \beta_t \sigma_{t-1}(x) + |f(x)|.
$$
Since $\sup_{x}|f(x)| \le \|f\|_{\mathcal{H}_k} \le B$, we have that $ \mu_{t-1}(x) + \beta_t \sigma_{t-1}(x) + B \ge 0$.
\end{proof}
We next provide justifications for the nonnegativity of the TS acquisition function after the constant shift. To this end, we show the existence of a constant $C_t$ such that $\alpha_t + C_t$ is non-negative on the search space $\cX_t$.
\begin{lemma}
    For some $\delta > 0$, the constant shifted $t^{\text{th}}$ acquisition function of the robust GP-TS algorithm, given by
    $$
        \alpha_t(x) + C_t = f_{t}(x)  + C_t
    $$
    is non-negative on $\cX_t$ with probability $1-\delta$, for some sequence $C_t$. 
    \end{lemma}
\begin{proof}
From Lemma 5 of \citep{chowdhury2017kernelized}, conditioning on the history until time $t$, i.e., $\cH_t:=\{(x_1, y_1), \cdots, (x_t, y_t)\}$, we know that
$$
\forall x \in \mathcal{X}_t, \quad |f_t(x)-\mu_{t-1}(x)| \le \sqrt{2\log(|\mathcal{X}_t|/\delta)}s_{t-1}(x),
$$
with probability $1-\delta/2$. Therefore, from the definition $s_{t-1}(x) = \beta_t \sigma_{t-1}(\cdot) + v_t$, we have
\begin{equation}\label{lem12:first}
\mu_{t-1}(x) - \sqrt{2\log(|\mathcal{X}_t|/\delta)}(\beta_T + v) \le \mu_{t-1}(x) - \sqrt{2\log(|\mathcal{X}_t|/\delta)}s_{t-1}(x) \le f_t(x),
\end{equation}
where $v_t = \left(\frac{1}{\eta_t}-1\right)B$. Using Lemma \ref{lem:conf} with $\beta_t\sigma_t$ replaced with $s_t$, with probability at least $1-\delta/2$
\begin{equation}\label{lem12:second}
-B - (\beta_T+v) \le f(x) - s_{t-1}(x)\le\mu_{t-1}(x)
\end{equation}
holds for all $x \in \cX$.
Combining \eqref{lem12:first} and \eqref{lem12:second}, we observe that
$$
f_t(x) + (1+\sqrt{2\log(|\mathcal{X}_t|/\delta)})(\beta_T + v) + B \ge 0,
$$
for all $x \in \cX_t$ with probability at least $1-\delta$.
\end{proof}

\begin{theorem}[Restatement of Theorem \ref{thm:GP_UCB_MFC}]\label{append_thm:GP_UCB_MFC}
Under assumptions \ref{append_assump:kernel_less_one} and \ref{append_assump:kernel_smooth}, suppose an objective function $f \in \mathcal{H}_k$ with $\|f\|_{\mathcal{H}_k} \le B$. The inexact GP-UCB algorithm with $\beta_t = B + R\sqrt{2(\gamma_{t-1} + 1 + \log(1/\delta))}$ and worst-case accumulated inaccuracy $M_T$ yields a cumulative regret bound of the form,
$$
R_T = \mathcal{O}\left(\gamma_T\sqrt{T}  + M_T\sqrt{\gamma_T} \right),
$$ 
with probability $1-\delta$.
\end{theorem}

\begin{proof}
From Theorem 2 of \citet{chowdhury2017kernelized}, we know that
\begin{align*}
     \mu_{t-1}(x^*) - \beta_t \sigma_{t-1}(x^*) &\le f(x^*) \le \mu_{t-1}(x^*) + \beta_t \sigma_{t-1}(x^*), \\
      \mu_{t-1}(x_t) - \beta_t \sigma_{t-1}(x_t) &\le f(x_t) \le \mu_{t-1}(x_t) + \beta_t \sigma_{t-1}(x_t)
\end{align*}
hold with probability $1-\delta$. Then
\begin{align*}
    f(x^*) - f(x_t) &\le \mu_{t-1}(x^*) + \beta_t \sigma_{t-1}(x^*) - \mu_{t-1}(x_t) + \beta_t \sigma_{t-1}(x_t) \\
    &= \alpha_{t-1}(x^*) - \alpha_{t-1}(x_t) + 2\beta_t \sigma_{t-1}(x_t) \\
    &= (1-\eta_t) \alpha_t^* + 2\beta_t \sigma_{t-1}(x_t),
\end{align*}
where the first equality follows from the definition of the UCB acquisition function, and the second equality is due to the fact that $\alpha_{t}(x_t) = \eta_t \alpha_t(x^*)$.
Since $\mu_{t-1}(x) \le f(x) + \beta_t \sigma_{t-1}(x)$, for all $x \in \mathcal{X}$, we conclude that, with probability $1-\delta$,
\begin{align*}
\alpha_t^* \le \max_{x\in\mathcal{X}} [f(x) + 2\beta_T \sigma_{t-1}(x) ]\le B + 2\beta_T,
\end{align*}
where the second inequality follows from the fact that $\sup_{x \in \mathcal{X}}|f(x)| \le \|f\|_{\mathcal{H}_k} \le B$, and for all $x \in \mathcal{X}, \sigma_{t-1}(x) \le 1$. Therefore, we conclude that
\begin{align*}
    R_T &= \sum_{t=1}^T f(x^*) - f(x_t) 
    \le (B+2\beta_T) \sum_{t=1}^T (1-\eta_t) + 2\beta_T \sum_{t=1}^T \sigma_{t-1}(x_t)  \le (B+2\beta_T) \sum_{t=1}^T (1-\tilde\eta_t) + 2\beta_T \sum_{t=1}^T \sigma_{t-1}(x_t).
\end{align*}
From Lemma \ref{lem:sigma_bound}, we arrive at the conclusion.
\end{proof}

A similar analysis can be established for the GP-TS algorithm with an enlarged variance $s^2_{t-1}(x) = \left(\beta_t \sigma_{t-1}(\cdot) + v_t\right)^2$, $v_t = (\frac{1}{\tilde \eta_t}-1)B$. We omit the proof of the following statements as it substantially overlaps with the approach taken in \citep{chowdhury2017kernelized} with a variance factor adjustment.  

\begin{theorem}[Restatement of Theorem \ref{thm:GP_TS_MFC}]\label{append_thm:GP_TS_MFC}
Under assumptions \ref{append_assump:kernel_less_one} and \ref{append_assump:kernel_smooth}, suppose an objective function $f \in \mathcal{H}_k$ with $\|f\|_{\mathcal{H}_k} \le B$. The inexact GP-TS algorithm with variance $s^2_{t-1}(x) = \left(\beta_t \sigma_{t-1}(\cdot) + \tilde v_t\right)^2$, $\tilde v_t = (\frac{1}{\tilde \eta_t}-1)B$ and worst-case accumulated inaccuracy $M_T$ yields a cumulative regret bound of the form
$$
R_T = \mathcal{O}\left(\gamma_T\sqrt{T \log T} + M_T \sqrt{\gamma_T\log T}\right),
$$ 
with probability $1-\delta$.
\end{theorem}

\subsection{Random Grid Search based GP-UCB}\label{sec:aux}
\begin{theorem}\label{append_thm:random_UCB}
Under assumptions \ref{append_assump:kernel_less_one}, \ref{append_assump:kernel_smooth} and \ref{append_assum:grid}, suppose $f \in \mathcal{H}_k$ with $\|f\|_{\mathcal{H}_k} \le B$. With probability at least $1-\delta$, the random grid GP-UCB with $\beta_t = B + R\sqrt{2(\gamma_{t-1} + 1 + \log(2/\delta))}$ yields 
\begin{align*}
R_T = \mathcal{O}\left(\gamma_T \sqrt{T} + T^{\frac{d-1}{d} + \xi}\right),
\end{align*}
for arbitrarily small $\xi > 0$.
\end{theorem}
\begin{proof}
Note that 
\begin{align*}
R_T &= \sum_{t=1}^T f(x^*) - f(x_t)\\
&= \sum_{t=1}^T f(x^*) - f([x^*]_t) + f([x^*]_t)  - f(x_t),
\end{align*}
where $[x^*]_t$ is the point closest to $x^*$ in $\mathcal{X}_t$. By Lemma \ref{lem:conf}, for all possible realizations of $\mathcal{X}_1, \cdots, \mathcal{X}_T$, with a probability $1-\delta/2$, we have
\begin{align*}
\mu_{t-1}([x^*]_t) - \beta_t \sigma_{t-1}([x^*]_t) &\leq f([x^*]_t) \leq \mu_{t-1}(x^*_t) + \beta_t \sigma_{t-1}([x^*]_t) \\
\mu_{t-1}(x_t) - \beta_t \sigma_{t-1}(x_t) &\leq f(x_t) \leq \mu_{t-1}(x_t) + \beta_t \sigma_{t-1}(x_t).
\end{align*}
From the definition of $x_t$, we know $\mu_{t-1}([x^*]_t) + \beta_t \sigma_{t-1}([x^*]_t) \le \mu_{t-1}(x_t) + \beta_t \sigma_{t-1}(x_t)$. Therefore, we have
$$
f([x^*]_t) - f(x_t) \le 2\beta_t \sigma_{t-1}(x_t).
$$
Furthermore, from Lemma 1 of \citep{de2012regret}, we know that $f(x^*) - f([x^*]_t) \le C \|x^*-[x^*]_t\|$ for some constant $C > 0$. Since $\|x^*-[x^*]_t\| \le h_t$ for $h_t := \sup_{x \in \cX} \inf_{\bar x_i \in \cX_t}\|x-\bar x_i\|$, we have
$$
R_T \le C\sum_{t=1}^T h_t + 2\beta_T\sum_{t=1}^T\sigma_{t-1}(x_t).
$$
Then by Lemma \ref{lemma:grid_disc}, we know that with probability at least $1-\delta/2$, $C\sum_{t=1}^T h_t = \mathcal{O}\left(T^{\frac{d-1}{d} + \xi}\right)$. Combined with Lemma \ref{lem:sigma_bound} and invoking the union bound, we have the result.
\end{proof}

\subsection{Random Grid Search based GP-TS}\label{sec:aux}
In this section, we list all additional definitions and lemmas we use in proofs for the regret bound of a random grid search based GP-TS. Since our approach closely follows to that of \citet{chowdhury2017kernelized}, many of the preliminary lemma we list here can be proven in an analogous fashion. We will adjust and restate the Lemma and proof if needed.

\begin{definition}
We define the filtration $\mathcal{F}_t = \sigma\left\{(x_1, y_1, \cX_2), \cdots, (x_t, y_t, \cX_{t+1})\right\}$ as the $\sigma$-algebra generated by the collection of evaluation locations, function evaluations and search grids observed until time $t$. Note that conditional on $\cF_t$, the search grid is no longer random.
\end{definition}

\begin{lemma}[Lemma 5 of \citep{chowdhury2017kernelized}]\label{lem:thompson_gap}
For all $t \in \mathbb{N}$, assume $|\mathcal{X}_t| = ct$ with $c > 0$. Then
$$
\mathbb{P}\left[\forall x \in \mathcal{X}_t, |f_t(x) - \mu_{t-1}(x)| \le \beta_t\sqrt{2\log(ct^3)}\sigma_{t-1}(x) \Big|\mathcal{F}_{t-1} \right] \ge 1 - 1/t^2,
$$
for some constant $c > 0$, $\delta \in (0,1)$ and $\beta_t = B + R\sqrt{2(\gamma_{t-1} + 1 + \log(3/\delta))}$.
\end{lemma}

\noindent We introduce some definitions here.
\begin{definition}\label{def:CT}
$\forall t \ge 1$, $\tilde c_t = \sqrt{6\log t + 2\log c}$ and $c_t = \beta_t(1+\tilde c_t)$.
\end{definition}

\begin{definition}
We define the following two events:
\begin{align*}
E^f(t) &= \left\{\forall x \in \mathcal{X}, |\mu_{t-1}(x)-f(x)| \le \beta_t\sigma_{t-1}(x) \right\} \\
E^{f_t}(t) &= \left\{\forall x \in \mathcal{X}_t, |f_t(x)-\mu_{t-1}(x)| \le \beta \tilde c_t 
\sigma_{t-1}(x)\right\}
\end{align*}
\end{definition}

\begin{definition}
Given a grid $\mathcal{X}_t$, define the set of saturated points to be
    $$
    S_t \coloneqq \{x\in \mathcal{X}_t: \Delta_t(x) > c_t \sigma_{t-1}(x)\},
    $$
    where $[x^*]_t$ is the point closest to $x^*$ in $\mathcal{X}_t$ and $\Delta_t(x) : = f([x^*]_t)-f(x)$. Notice that conditioning on $\mathcal{X}_t$, $[x^*]_t \in \mathcal{X}_t \setminus S_t$.
\end{definition}

\begin{lemma}[Lemma 6 of \citep{chowdhury2017kernelized}]\label{lemma::event_prob}
    Suppose $|\mathcal{X}_t| = ct, ~c > 0$. For $\delta \in (0,1)$, following statements hold.
    \begin{itemize}
        \item $\mathbb{P}[\forall~ t \ge 1, E^f(t)] \ge 1 - \delta/3$
        \item $\mathbb{P}[E^{f_t}(t)| \mathcal{F}_{t-1}] \ge 1-1/t^2$
    \end{itemize}
\end{lemma}

\begin{lemma}[Lemma 7 of \citet{chowdhury2017kernelized}]\label{lemma::prob_lower}
    For any filtration $\mathcal{F}_{t-1}$ such that $E^f(t)$ is true,
    $$
    \mathbb{P}\left[f_t(x) > f(x)\Big|\mathcal{F}_{t-1}\right] \ge \eta\coloneqq \frac{1}{4e\sqrt{\pi}} > 0,
    $$
    holds for any $x \in \mathcal{X}$.
\end{lemma}

\begin{lemma}[Lemma 8 of \citet{chowdhury2017kernelized}]\label{lemma:prob_good_event}
   For any filtration $\mathcal{F}_{t-1}$ such that $E^f(t)$ is true,
    $$
    \mathbb{P}[x_t \in \mathcal{X}_t \setminus S_t| \mathcal{F}_{t-1}] \ge \eta - 1/t^2.
    $$
\end{lemma}

\begin{lemma}\label{lemma::exp_reg_bound}
    For any filtration $ \mathcal{F}_{t-1}$ such that $E^f(t)$ is true, we have
    \begin{align*}
    \mathbb{E}[\Delta_t(x_t) | \mathcal{F}_{t-1}] \le \frac{11c_t}{\eta} \mathbb{E}[\sigma_{t-1}(x_t)|  \mathcal{F}_{t-1}] + \frac{2B}{t^2},
    \end{align*}
    where $\Delta_t(x_t) := f([x^*]_t)-f(x_t)$.
\end{lemma}
\begin{proof}
Given a grid $\mathcal{X}_t$, let $\bar x_t = \argmin_{x \in \mathcal{X}_t \setminus S_t}\sigma_{t-1}(x)$. From the law of total expectation and positivity of the $\sigma_{t-1}$, we have 
\begin{align}
\mathbb{E}[\sigma_{t-1}(x_t)| \mathcal{F}_{t-1}] &\ge \mathbb{E}[\sigma_{t-1}(x_t)| \mathcal{F}_{t-1}, x_t \in \cX_t \setminus S_t]\mathbb{P}[x_t \in \cX_t \setminus S_t| \cF_{t-1}] \nonumber \\ &\ge \mathbb{E}[\sigma_{t-1}(\bar x_t)|\cF_{t-1}] (\eta-1/t^2), \label{ineq::lower_sigma}
\end{align}
where the second inequality follows from the definition of $\bar x_t$ and Lemma \ref{lemma:prob_good_event}. 

To control $\Delta_t(x_t) = f([x^*]_t) - f(x_t)$, recall that if $E^f(t)$ and $E^{f_t}(t)$ are both true, we know that
\begin{equation}\label{eq::up_down_f}
  \text{for all } x \in \mathcal{X}_t, \quad f_t(x) - c_t \sigma_{t-1}(x) \le f(x) \le f_t(x) + c_t \sigma_{t-1}(x).  
\end{equation}
Notice that
\begin{align*}
\Delta_t(x_t) 
&= f([x^*]_t) - f(x_t) \\
&= f([x^*]_t) - f(\bar x_t) +  f(\bar x_t) - f(x_t)\\
&\le \Delta_t(\bar x_t) + f_t(\bar x_t) + c_t \sigma_{t-1}(\bar x_t) - f_t(x_t) +  c_t \sigma_{t-1}(x_t) \\
&\le  c_t (2\sigma_{t-1}(\bar x_t) + \sigma_{t-1}(x_t)) + f_t(\bar x_t) - f_t(x_t)\\
&\le  c_t (2\sigma_{t-1}(\bar x_t) + \sigma_{t-1}(x_t)) 
\end{align*}
where the first inequality is due to \eqref{eq::up_down_f}, the second inequality follow from the fact $\bar x_t \notin S_t$ and the last inequality comes from the definition of $x_t$.

Therefore, we have
\begin{align*}
\mathbb{E}[\Delta_t(x_t)|\cF_{t-1}] 
&\le 2c_t \mathbb{E}[\sigma_{t-1}(\bar x_t) |\cF_{t-1}] + c_t \mathbb{E}[\sigma_{t-1}(x_t) |\cF_{t-1}] + 2B \mathbb{P}\left[E^{f_t}(t)^c|H_{t-1}\right]\\
&\le \frac{2c_t}{\eta-1/t^2} \mathbb{E}[\sigma_{t-1}(x_t) |\cF_{t-1}] + c_t \mathbb{E}[\sigma_{t-1}(x_t) |\cF_{t-1}] + \frac{2B}{t^2} \\
&\le \frac{11c_t}{\eta} \mathbb{E}[\sigma_{t-1}(x_t) |\cF_{t-1}] + \frac{2B}{t^2},
\end{align*}
where we used the fact $\sup_{x \in cX}\Delta_t(x) \le 2B$ which can be deduced from the fact $\sup |f(x)| \le \|f\|_{\cH_k} \le B$ in the first inequality. The second inequality follows from the inequality in \eqref{ineq::lower_sigma} and Lemma \ref{lem:thompson_gap}. The third inequality is due to the fact $\frac{1}{\eta-1/t^2} \le \frac{5}{\eta}$.
\end{proof}

\noindent Next, we define random variables and associated filtration to invoke concentration inequality for super-martingales.
\begin{definition}\label{def::sup_martingale}
Let $Y_0 = 0$, and for all $t \in \{1, \cdots, T\}$,
     \begin{align*}
    \bar{\Delta}_t(x_t) &= \Delta_t(x_t) \mathbb{I}\left\{E^f (t)\right\} \\
    Z_t &= \bar{\Delta}_t(x_t) -\frac{11c_t}{\eta} \sigma_{t-1}(x_t) - \frac{2B}{t^2}\\ 
    Y_t &= \sum_{s=1}^t Z_s
\end{align*}
 \end{definition}

\noindent From the definition, and by Lemma \ref{lemma::exp_reg_bound}, we deduce the following result, which we formally state as lemma.

\begin{lemma}\label{lemma::sup_martingale}
$(Y_t)_{t=0}^T$ is a super-martingale process with respect to filtration $\mathcal{F}_{t}$.
\end{lemma}

\begin{proof}
It suffices to show that for all $t \in \{1, \cdots, T\}$ and any $\mathcal{F}_{t-1}$, $\mathbb{E}\left[Y_t - Y_{t-1} | \cF_{t-1} \right] \le 0$. Note that
\begin{equation}\label{eq::super_martingale}
\mathbb{E}\left[Y_t - Y_{t-1} | \cF_{t-1} \right] = \mathbb{E}\left[Z_t | \cF_{t-1} \right] = \mathbb{E}\left[ \bar \Delta_t(x_t)| \cF_{t-1}\right] - \frac{11c_t}{\eta}\mathbb{E}\left[\sigma_{t-1}(x_t)| \cF_{t-1}\right] - \frac{2B}{t^2}.
\end{equation} 
If $E^t(t)$ is false, we have $\mathbb{E}\left[ \bar \Delta_t(x_t)| \cF_{t-1}\right] = 0$, which shows that \eqref{eq::super_martingale} is less than or equal to zero. On the other hand, if $E^t(t)$ is true, from Lemma \ref{lemma::exp_reg_bound}, we can again conclude that \eqref{eq::super_martingale} is less than or equal to zero.
\end{proof}

\begin{lemma}\label{lemma:discrep_TS_bound}
Given any $\delta > 0$,
\begin{align*}
    \sum_{t=1}^T \Delta_t(x_t) &\le \frac{11c_T}{\eta}\sum_{t=1}^T\sigma_{t-1}(x_t) + \frac{2B\pi^2}{6} + \frac{4B+11c_T}{\eta}\sqrt{ 2T \log(3/\delta) },
\end{align*}
 with probability at least $1-2\delta/3$.
\end{lemma}

\begin{proof}
By construction, 
$$
|Y_t-Y_{t-1}| = |Z_t| \le \left|\bar\Delta_t(x_t)\right| + \frac{11c_t}{\eta} \sigma_{t-1}(x_t) + \frac{2B}{t^2} \le 2B + \frac{11c_t}{\eta} + \frac{2B}{t^2} \le \frac{4B + 11c_t}{\eta}
$$
where the first inequality is due to the triangle inequality. The second inequality comes from the fact that $\left|\bar\Delta_t(x_t)\right| \le 2\sup_{x \in \mathcal{X}} |f(x)| \le 2\|f\|_{\cH_k} \le 2B$ and $\sigma_{t-1}(x) \le 1$ for all $x \in \mathcal{X}$. The third inequality follows from $\eta \le 1$. From the Azuma-Hoeffding inequality, with at least probability $1-\delta/3$, we have
\begin{align*}
    Y_T - Y_0 = \sum_{t=1}^T \bar \Delta_t(x_t) -\sum_{t=1}^T \frac{11c_t}{\eta} \sigma_{t-1}(x_t) - \sum_{t=1}^T \frac{2B}{t^2} \le \sqrt{2\log(3/\delta)\sum_{t=1}^T \frac{(4B+11c_t)^2}{\eta^2}}.
\end{align*}
In other words, we have
\begin{align*}
\sum_{t=1}^T \bar \Delta_t(x_t) &\le \sum_{t=1}^T \frac{11c_t}{\eta} \sigma_{t-1}(x_t) + \sum_{t=1}^T \frac{2B}{t^2} + \sqrt{2\log(3/\delta)\sum_{t=1}^T \frac{(4B+11c_t)^2}{\eta^2}} \\
&\le \frac{11c_T}{\eta}\sum_{t=1}^T  \sigma_{t-1}(x_t) + \frac{2B \pi^2}{6} + \frac{4B+11c_T}{\eta}\sqrt{ 2T \log(3/\delta) } 
\end{align*}
with at least probability $1-\delta/3$. From Lemma \ref{lemma::event_prob}, we know $E^f(t)$ holds for all $t \ge 1$ with probability at least $1- \delta/3$. In other words, by definition, $\Delta_t(x_t) = \bar \Delta_t(x_t)$ for all $t \ge 1$ with probability at least $1- \delta/3$. Applying the union bound, we obtain the statement.
\end{proof}

\begin{theorem}[Restatement of Theorem \ref{thm:random_TS}]
Under assumptions \ref{append_assump:kernel_less_one}, \ref{append_assump:kernel_smooth} and \ref{append_assum:grid}, suppose $f \in \mathcal{H}_k$ with $\|f\|_{\cH_k} \le B$, for some $B >0$. With probability at least $1-\delta$, the random grid GP-TS with $\beta_t = B + R\sqrt{2(\gamma_{t-1} + 1 + \log(3/\delta))}$ yields 
\begin{align*}
R_T = \mathcal{O}\left(\gamma_T \sqrt{T \log T} + T^{\frac{d-1}{d} + \xi} \right),
\end{align*} 
for arbitrarily small $\xi > 0$.
\end{theorem}
\begin{proof}
Note that
\begin{align*}
R_T &= \sum_{t=1}^T f(x^*) - f(x_t) \\
&= \sum_{t=1}^T f(x^*) - f([x^*]_t) + f([x^*]_t) - f(x_t) \\
&= \sum_{t=1}^T f(x^*) - f([x^*]_t) + \sum_{t=1}^T \Delta_t(x_t) \\
&\le C \sum_{t=1}^T \|x^*-[x^*]_t\|_2 + \sum_{t=1}^T \Delta_t(x_t) \\
& \le C \sum_{t=1}^T h_t + \sum_{t=1}^T \Delta_t(x_t),
\end{align*}
where the first inequality follows from the Lipschitzness of $f$, which was shown in Lemma 1 of \citep{de2012regret} and the second inequality is due to the definition of fill distance $h_t = \sup_{x \in \mathcal{X}}\inf_{\tilde x_i \in \mathcal{X}_t} \|x-\tilde x_i\|$. From Lemma \ref{lemma:grid_disc}, we know that the leading term $C \sum_{t=1}^T h_t = \mathcal{O}\left(T^{\frac{d-1}{d} + \xi}\right)$. For the second term, note that from Lemma \ref{lemma:discrep_TS_bound}
\begin{align*}
    \sum_{t=1}^T \Delta_t(x_t) &\le \frac{11c_T}{\eta}\sum_{t=1}^T\sigma_{t-1}(x_t) + \frac{2B\pi^2}{6} + \frac{4B+11c_T}{\eta}\sqrt{ 2T \log(3/\delta) } \\
    &= \mathcal{O}\left(c_T \sqrt{4(T+2)\gamma_T} + c_T\sqrt{T}\right)
\end{align*}
where the last equality is due to Lemma
\ref{lem:sigma_bound}. From the definition of $c_t$, we know $c_T = \mathcal{O}(\sqrt{\gamma_T \log T})$. Since the leading term dominates, we have
$$
\sum_{t=1}^T \Delta_t(x_t) = \mathcal{O}\left(\gamma_T \sqrt{T \log T}\right).
$$
Applying union bound and combining everything, we get the result.
\end{proof}

\section{Additional Details of Experiments}\label{app:exp}

\subsection{Experimental Setup of Real-World AutoML Task}

In the 11-dimensional real-world hyperparameter tuning task, we focus on improving the validation accuracy of the gradient boosting model (from the Python \texttt{sklearn} package) on the breast-cancer dataset (from the UCI machine learning repository). The 11 hyperparameters are:

\begin{itemize}
    \item Loss, (string, “logloss” or “exponential”).
    \item Learning rate, (float, (0, 1)).
    \item Number of estimators, (integer, [20, 200]).
    \item Fraction of samples to be used for fitting the individual base learners, (float, (0, 1)).
    \item Function to measure the quality of a split, (string, “friedman mse” or “squared error”).
    \item Minimum number of samples required to split an internal node, (integer, [2, 10]).
    \item Minimum number of samples required to be at a leaf node, (integer, [1, 10]).
    \item Minimum weighted fraction of the sum total of weights, (float, (0, 0.5)).
    \item Maximum depth of the individual regression estimators, (integer, [1, 10]).
    \item Number of features to consider when looking for the best split, (float, “sqrt” or “log2").
    \item Maximum number of leaf nodes in best-first fashion, (integer, [2, 10]).
\end{itemize}

For integer-valued parameters, we conducted optimization on the continuous domain and rounded to the nearest integer value.

\subsection{Ablation Studies}

\subsubsection{Experiments with Different Grid Sizes}
We have further added a numerical experiment investigating the effect of linearly varying grid size versus a fixed grid size of 100 in the 11-dimensional machine learning hyperparameter tuning problem. From Figure \ref{fig:fix}, it can be seen that the linearly varying grid size performs slightly better than the fixed grid size in terms of the cumulative regret plot.

\begin{figure}[!htbp]
    \centering
\begin{minipage}{0.5\linewidth}\centering
		\includegraphics[width=\textwidth]{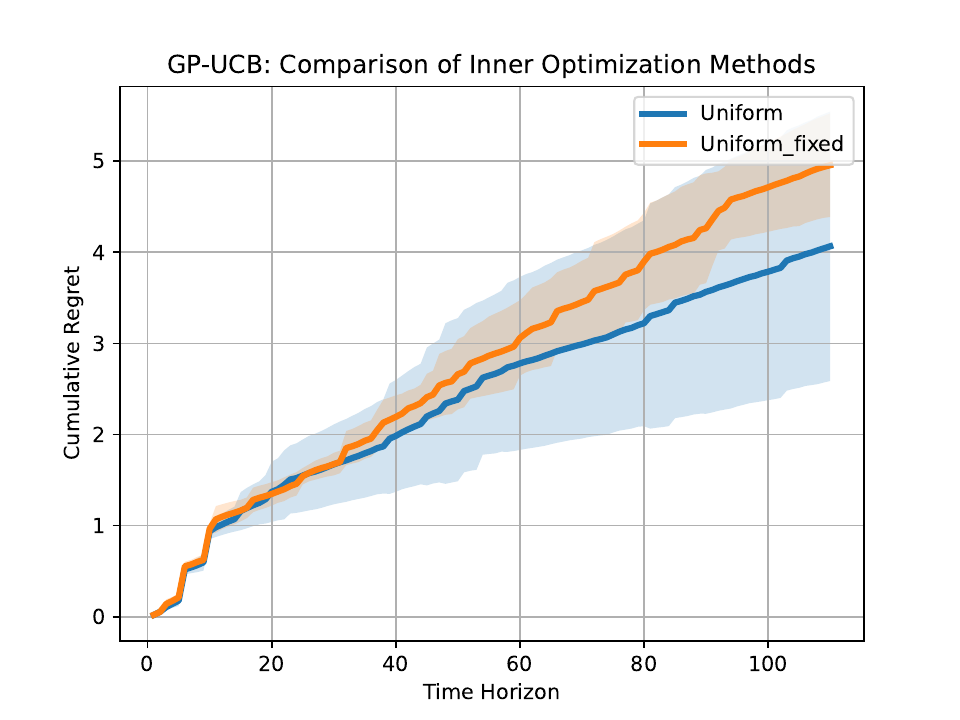}
	\end{minipage}
    \caption{Performances with uniform random and uniform fixed grid sizes.}
    \label{fig:fix}
\end{figure}

\subsubsection{Experiments with a Smaller Initial Design Points}
We further investigated the impact of a smaller initial design. In Section \ref{sec:exp}, we used $n = 10d$; here, we repeat those simulations with $n = 5d$. Results are shown in Figure \ref{fig:cum_reg_small} and Figure \ref{fig:comp_time_small}. Once again, the random grid search GP-UCB achieved competitive performance while substantially improving computational efficiency.

\begin{figure}[!htbp]
    \centering
    \begin{minipage}{0.5\linewidth}\centering
		\includegraphics[width=\textwidth]{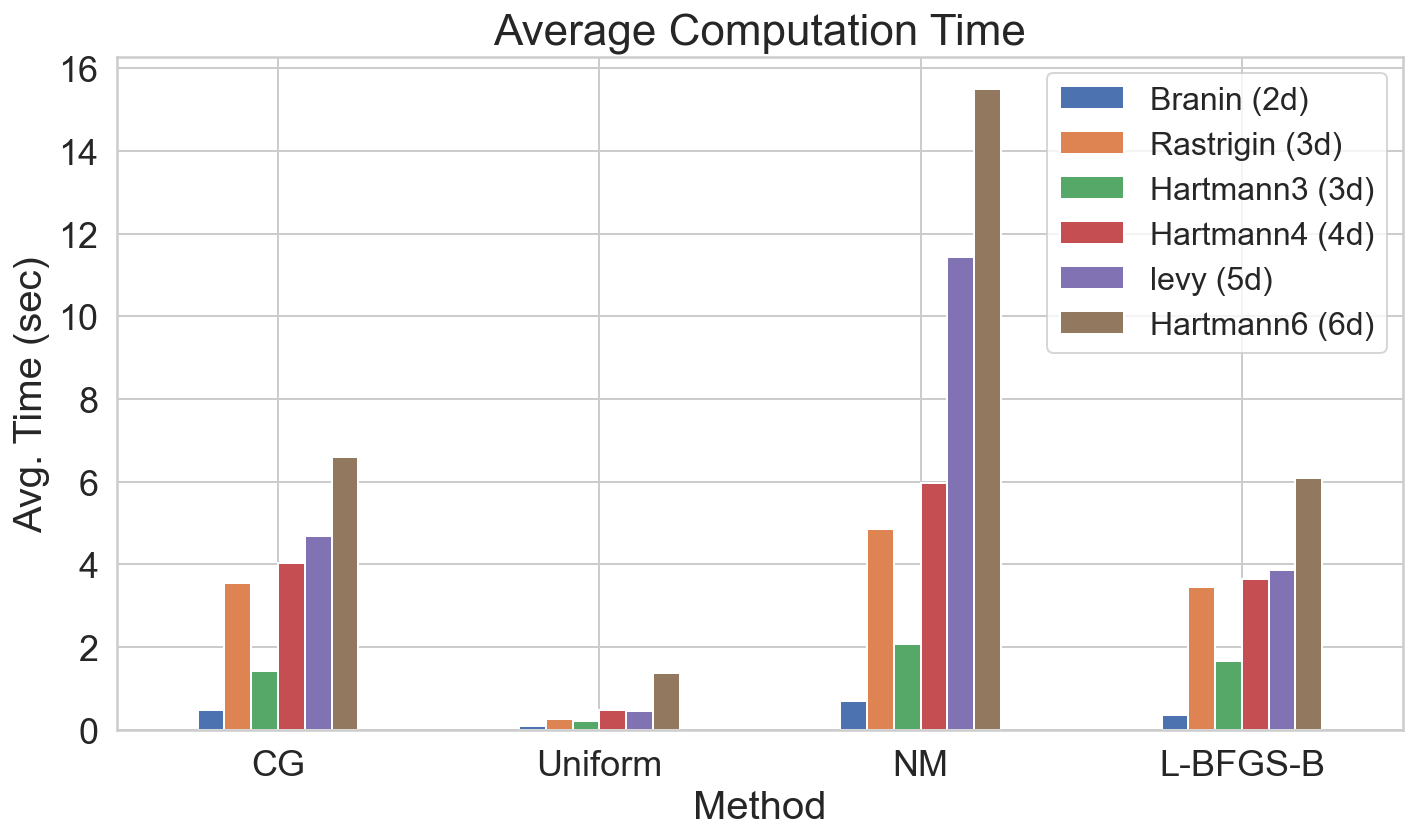}
	\end{minipage}
    \caption{Average computational time of different acquisition optimization methods with a smaller initial design points ($n = 5d$, in the manuscript, we considered $n = 10d$). }
    \label{fig:comp_time_small}
\end{figure}

\begin{figure}[!htbp]
    \centering
    \begin{minipage}{0.8\linewidth}\centering
		\includegraphics[width=\textwidth]{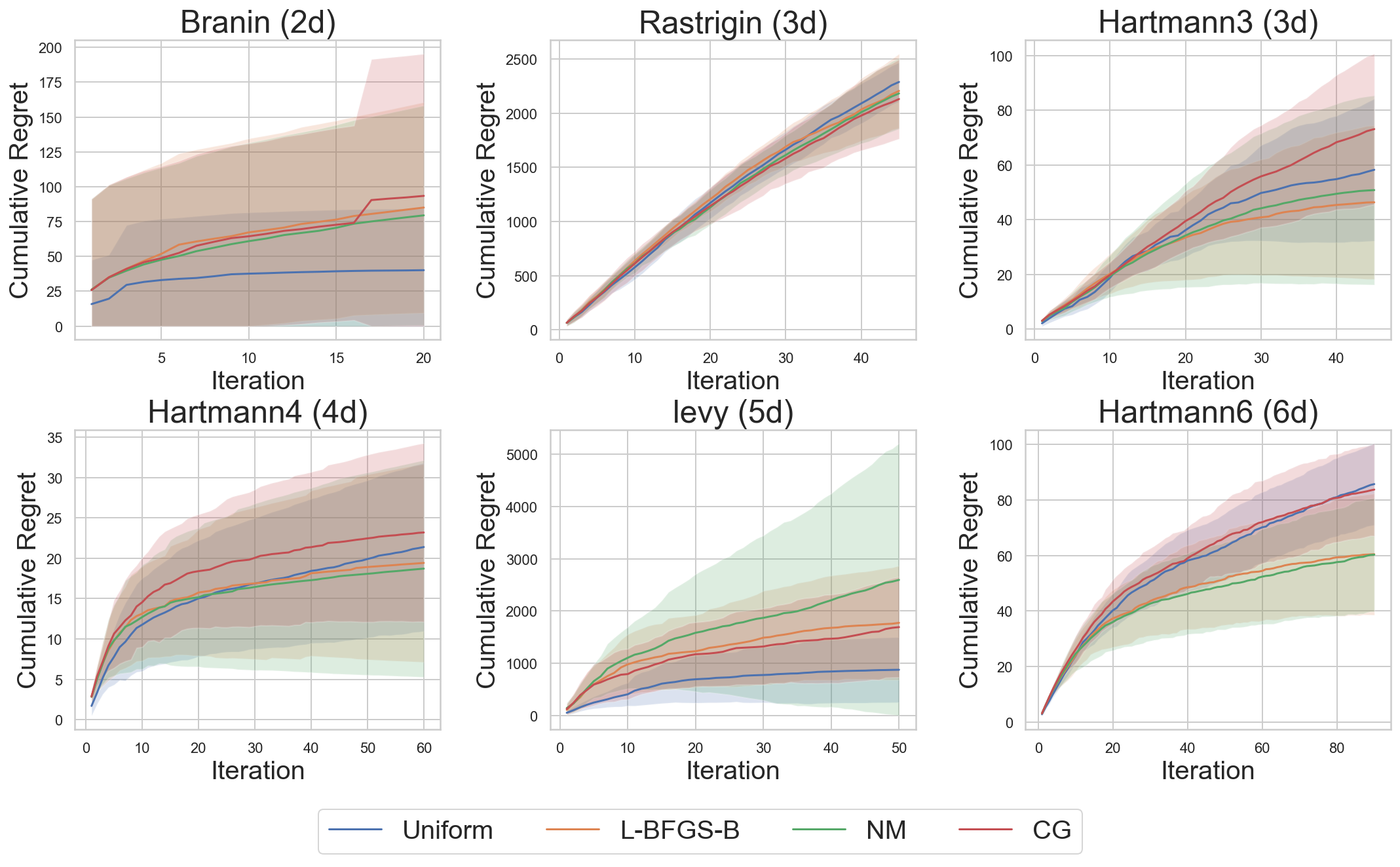}
	\end{minipage}
    \caption{Cumulative regret of different acquisition optimization methods with a smaller initial design points ($n = 5d$, in the manuscript, we considered $n = 10d$). }
    \label{fig:cum_reg_small}
\end{figure}



\end{document}